\documentclass[aos]{imsart}

\RequirePackage{amsthm,amsmath,amsfonts,amssymb}
\RequirePackage[authoryear]{natbib}
\RequirePackage{graphicx}
\RequirePackage{booktabs}
\RequirePackage{tabularx}
\RequirePackage{array}
\RequirePackage{xcolor}
\RequirePackage[colorlinks,citecolor=blue,urlcolor=blue]{hyperref}

\startlocaldefs

\theoremstyle{plain}

\newtheorem{proposition}{Proposition}
\theoremstyle{definition}

\theoremstyle{definition}

\endlocaldefs

\begin{document}

\begin{frontmatter}

\title{Amortized Bandwidth Learning for Kernel Density Estimation under Logarithmic Score}
\runtitle{Amortized KDE Bandwidth Learning}

\begin{aug}

\author[A]{\fnms{Junyi}~\snm{Liang}
\ead[label=e1]{y20240092@mail.ecust.edu.cn}}

\author[B,A,C]{\fnms{Hailiang}~\snm{Du}\thanksref{t1}
\ead[label=e2]{hailiang.du@durham.ac.uk}}

\address[A]{School of Mathematics,
East China University of Science and Technology,
200237 Shanghai, China\printead[presep={,\ }]{e1}}

\address[B]{Department of Mathematical Sciences;
Institute of Hazard, Risk and Resilience,
Durham University,
DH1 3LE Durham, United Kingdom\printead[presep={,\ }]{e2}}

\address[C]{Data Science Institute;
Global School of Sustainability,
The London School of Economics and Political Science,
WC2A 2AE London, United Kingdom}

\thankstext{t1}{Corresponding author.}

\end{aug}

\begin{abstract}
Kernel density estimation converts finite samples into probability densities,
but its performance depends critically on bandwidth selection. Classical
selectors prescribe the sample-to-bandwidth rule analytically or asymptotically,
or solve a new optimization for each sample. An amortized framework is proposed that instead learns this mapping across a distribution of density-estimation tasks by optimizing the logarithmic score. A truncated-and-renormalized bounded-support formulation enables stable
learning across heterogeneous tasks, while affine standardization allows
a selector trained on a single reference interval to transfer across bounded
intervals.  Experiments under Gaussian sampling, a multi-family benchmark, and
randomized Gaussian-mixture training show that the amortized selector
consistently and substantially outperforms Silverman's rule, the
Sheather--Jones selector, and least-squares cross-validation, with especially
large gains in small and heterogeneous samples. Finite Gaussian mixtures provide
a generic training mechanism supported by their $L^1$ approximation property. Selectors trained in this way generalize strongly across different
density structures, allowing the same trained selector to be applied
directly to finite samples from unknown densities without specifying or
fitting a distributional family. This combination of broad applicability
and strong empirical performance makes the framework attractive for a
wide range of applications in which finite samples or ensembles must be
converted into continuous probability densities.
\end{abstract}



\end{frontmatter}

\section{Introduction}

In many applications, one has access only to a finite sample or a finite ensemble from an underlying distribution, while the quantity ultimately needed for inference, evaluation, or decision-making is a continuous probability density. Examples include probabilistic forecasting, simulation-based uncertainty quantification, and other settings in which tail probabilities, predictive likelihoods, or calibrated distributional summaries are of primary interest. Kernel density estimation (KDE) is a classical nonparametric method for converting such finite samples into a density estimate~\citep{Silverman1986,Scott2015}. In kernel density estimation, the central practical difficulty is bandwidth selection, since the bandwidth controls the degree of smoothing and thereby has a decisive effect on estimation quality. The present study focuses on the widely used Gaussian-kernel setting. Although the methodology is developed and evaluated in this setting, the underlying idea of learning a reusable bandwidth rule is not restricted to the Gaussian kernel and can be extended to other kernel families.

A large literature has developed classical bandwidth selectors for kernel
density estimation, including rule-of-thumb methods
\citep{Silverman1986,Scott2015}, plug-in methods
\citep{SheatherJones1991,WandJones1995}, and cross-validation methods
\citep{Rudemo1982,Hall1983,Bowman1984}. Rule-of-thumb selectors, such as Silverman's rule, provide simple closed-form prescriptions based on normal-reference arguments together with global scale summaries of the sample~\citep{Silverman1986}, but this simplicity comes at the cost of relying on a fixed reference structure that may be poorly matched to skewed, heavy-tailed, contaminated, or multimodal densities. Plug-in methods are more flexible in that they typically begin from an asymptotic approximation to an estimation criterion and replace the resulting unknown population quantities, such as curvature-related functionals, by data-based estimates~\citep{SheatherJones1991,WandJones1995}.
Their performance can nevertheless depend materially on pilot-bandwidth and higher-order estimation choices~\citep{Loader1999}. Cross-validation methods rely less on an explicit reference model and instead select the bandwidth by optimizing a sample-based criterion over candidate values~\citep{Rudemo1982,Hall1983,Bowman1984}. These approaches differ in how the bandwidth is computed. Rule-of-thumb and plug-in procedures define reusable sample-to-bandwidth algorithms whose structure is prescribed analytically or asymptotically, whereas cross-validation procedures typically solve a new optimization problem for each observed sample. What they do not generally do is estimate the bandwidth-selection mapping itself from a distribution of related density-estimation tasks. Consequently, they cannot directly exploit systematic cross-task variation to learn which sample characteristics are predictive of good bandwidth choices under a specified evaluation criterion. This motivates a different formulation in which the bandwidth rule is learned across tasks. The proposed approach does not claim that classical selectors are unusable on new samples; rather, it replaces an analytically prescribed or repeatedly optimized selector with a mapping whose parameters are estimated from a collection of density-estimation problems. The resulting mapping remains adaptive to the observed sample, while also incorporating regularities identified across the training task distribution.

Classical bandwidth selection in kernel density estimation has largely been
developed under integrated squared-error criteria, especially mean integrated squared-error (MISE) and its
asymptotic approximations \citep{Silverman1986,Scott2015,WandJones1995}, although
Kullback--Leibler criteria have also been studied in this context
\citep{Hall1987}. These criteria are mathematically convenient, but KDE ultimately produces a
probability density, so bandwidth quality is naturally viewed as a problem of
probabilistic assessment. Strict propriety alone does not determine how imperfect
density estimates should be compared, since different strictly proper scoring rules
can rank them differently. Locality provides a further principle: for continuous
densities, the logarithmic score is the only proper local scoring rule, depending
directly on the density assigned to the realized outcome
\citep{Bernardo1979,BrockerSmith2007}. By contrast, nonlocal scores can
reward features of the reported density away from the realized outcome and may
produce unfortunate evaluations. The logarithmic score also has a direct
interpretation in terms of probabilities and bits of information, while relative
logarithmic score comparisons are invariant under smooth one-to-one
transformations of the variable \citep{Du2021}. For these reasons, the proposed
framework learns the bandwidth rule directly under the logarithmic score.

These considerations suggest a different formulation of the KDE bandwidth problem. Rather than fixing the functional form of the selector through an analytical or asymptotic prescription, or solving a new bandwidth-optimization problem for every observed sample, the present work learns the sample-to-bandwidth mapping from a distribution of density-estimation tasks. The core idea is to learn a shared sample-to-bandwidth mapping that extracts structural information from each sample and outputs a bandwidth adapted to that task. The resulting mapping, referred to as the amortized bandwidth selector, is reusable in the same operational sense as a classical selector, but its functional behaviour is estimated from cross-task data rather than prescribed in advance. Since the estimated object is a probability density, the amortized selector is trained and evaluated under the logarithmic score rather than an integrated squared-error criterion. To prevent the bandwidth from being driven to excessively large values by rare realizations in remote low-density regions under logarithmic score evaluation, the framework adopts a bounded-support formulation based on truncation and renormalization. This restricts both the target density and the KDE to a common interval, so that the learning objective is concentrated on the region of practical interest rather than being influenced by behaviour outside it. In this way, bandwidth selection is no longer treated as a fixed analytic prescription, but as a reusable rule learned directly from task variation.

The resulting framework is evaluated from three complementary perspectives. A Gaussian benchmark is first used to test whether the amortized selector captures the basic finite-sample and scaling behaviour of KDE bandwidths in the most transparent setting. A bounded multi-family benchmark is then introduced to
assess whether a single amortized selector remains effective under substantial structural variation across density families. Finally, a Gaussian mixture model (GMM) benchmark examines whether a flexible Gaussian-mixture generator can replace the designed multi-family distribution as a source of tasks for learning the amortized selector. Together, these experiments test whether the amortized selector improves empirical performance relative to classical selectors and remains effective across diverse density-estimation tasks.

The remainder of the paper is organized as follows. Section~2 introduces the proposed bandwidth-learning framework, including the amortized sample-to-bandwidth mapping, the logarithmic score objective, the bounded-support formulation, and the interval-standardization rule for cross-interval transfer. Section~3 presents the experimental results. Section~4 concludes with a summary of the main findings and a discussion of limitations and future directions.

\section{Bandwidth-Learning Framework}

The proposed framework retains KDE as the final density estimator and learns only the bandwidth-selection rule. Each observed sample is represented by a permutation-invariant set of features. A neural network maps these features to a bandwidth, which is then used to construct the KDE. Unlike classical bandwidth selectors, the sample-to-bandwidth mapping is learned jointly across density-estimation tasks and then reused for new samples without task-specific reoptimization. 

The framework is developed in three parts. Bandwidth selection is first formulated as an amortized learning problem, in which a shared predictor maps sample features to bandwidths across related tasks. The training criterion is then specified as the logarithmic score, aligning the selector with the
probabilistic objective used to assess density estimation. Finally, a truncated-and-renormalized bounded-support formulation is introduced together with affine standardization, which maps bounded task domains to a common reference interval and enables cross-interval transfer through bandwidth rescaling.

\subsection{Amortized Bandwidth Prediction}

This subsection formalizes the object that is learned. Let $f$ denote an unknown density on $\mathbb{R}$, and let
\begin{equation}
S=(x_1,\dots,x_n)\sim f^n
\end{equation}
be an observed i.i.d.\ sample of size $n$. For a bandwidth $h>0$, the associated Gaussian-kernel density estimator is
\begin{equation}
q_h(x;S)=\frac{1}{nh}\sum_{i=1}^n
\phi\!\left(\frac{x-x_i}{h}\right),
\qquad
\phi(u)=(2\pi)^{-1/2}e^{-u^2/2}.
\end{equation}
The amortized selector treats the sample-to-bandwidth relationship as an unknown
mapping to be learned across density-estimation tasks. Once trained, it can be
applied to new samples without task-specific reoptimization. The bandwidth rule
is written as
\begin{equation}
h_\theta(S)=F_\theta(\mathrm{feat}(S)),
\end{equation}
where $\mathrm{feat}(S)$ is a finite-dimensional vector of permutation-invariant sample features, $F_\theta$ is a parametric map with positive output, and $\theta$
denotes its trainable parameters. In the present proof-of-concept implementation, $F_\theta$ is realized by a simple multilayer perceptron (MLP), and the feature map consists of the sample
size, sample mean, sample standard deviation, sample skewness, and sample kurtosis:
\begin{equation}
\mathrm{feat}(S)=\bigl(n,\mathrm{mean}(S),\mathrm{sd}(S),
\mathrm{skew}(S),\mathrm{kurt}(S)\bigr).
\end{equation}
These features capture basic information on sample size, location, scale, asymmetry, and tail behaviour while remaining low-dimensional and interpretable. For the Gaussian diagnostic benchmark below, a simplified special case is used in which the network learns only a dimensionless bandwidth ratio as a function
of sample size.

\subsection{Logarithmic Score Criterion}

The training criterion is taken from the framework of proper scoring rules.
A scoring rule assigns a numerical loss to a reported probability density and an
observed outcome, with lower values indicating better probabilistic performance.
Strict propriety ensures that reporting the data-generating distribution is optimal
in expectation
\citep{GneitingEtAl2007,GneitingRaftery2007,BrockerSmith2007}.
This property is essential, but it does not uniquely determine how imperfect
density estimates should be compared: different strictly proper scoring rules can
rank them differently \citep{Du2021}.

A further distinction is provided by locality. A local scoring rule depends only
on the density assigned to the realized outcome. For continuous densities, the
logarithmic score is, up to affine equivalence, the only proper local scoring rule
\citep{Bernardo1979,BrockerSmith2007,JewsonRossell2022}. This property
is particularly relevant to bandwidth selection, since nonlocal scores can reward
features of the estimated density away from the realized outcome and may therefore
produce ``unfortunate'' evaluations
\citep{Du2021}. They may also change their preferences under smooth
transformations of the variable. By contrast, the logarithmic score directly
evaluates the density assigned to the realized value and has a direct interpretation
in terms of probabilities and bits of information
\citep{Good1952,RoulstonSmith2002}. Its sample average coincides with
the negative log-likelihood per observation, while its expectation is the
cross-entropy.

These considerations motivate the use of the logarithmic score as the training and evaluation criterion for bandwidth selection. For an independent realization $X\sim f$, the task-specific logarithmic score criterion is
\begin{equation}
\mathcal{L}(h;S,f)
=
\mathbb{E}_{X\sim f}
\left[-\log_2 q_h(X;S)\right]
=
-\int
\log_2 q_h(x;S)f(x)\,dx .
\end{equation}
To formalize learning across density-estimation tasks, let $\Pi$ denote the task distribution over $(f,n)$. The expected objective of the amortized selector over the task distribution is 
\begin{equation}
\mathcal{J}_{\Pi}(\theta)
=
\mathbb{E}_{(f,n)\sim\Pi}
\mathbb{E}_{S\sim f^n}
\left[
\mathcal{L}\bigl(h_\theta(S);S,f\bigr)
\right].
\end{equation}
In practice, the expectation in $\mathcal{J}_{\Pi}(\theta)$ is approximated using a finite sample of training tasks:
\begin{equation}
\widehat{\mathcal{J}}_{N}(\theta)
=
\frac{1}{N}
\sum_{j=1}^{N}
\mathcal{L}\bigl(h_\theta(S_j);S_j,f_j\bigr),
\qquad
(f_j,n_j)\sim\Pi,
\quad
S_j\sim f_j^{\,n_j}.
\end{equation}

\subsection{Bounded-Support Formulation and Cross-Interval Transfer}\label{sec:bounded_support}

A common bounded interval serves several purposes in the present learning framework. First, under logarithmic score evaluation, rare realizations in remote low-density regions can exert a disproportionate influence on the training objective and drive the learned bandwidth towards excessive smoothing. Restricting
the criterion to a region of practical interest limits this effect and focuses learning on the part of the distribution that is relevant for evaluation. Second, affine standardization maps bounded task domains to a common reference interval. This places sample features on comparable scales across tasks, so the predictor need not learn separately how to accommodate arbitrary
shifts and changes of scale. The bounded formulation also aligns naturally with applications in which the variable is physically constrained or only relevant over a finite range. For such variables, a Gaussian-kernel KDE may assign probability mass outside the admissible region; truncation and renormalization instead define the estimated density on the chosen interval.

For any density $f$ with positive mass on $[A,B]$, let $f_{[A,B]}(x)=f(x)\mathbf{1}_{[A,B]}(x)/\int_A^B f(u)\,du$ denote its truncated-and-renormalized version on that interval. The Gaussian KDE is truncated and renormalized analogously:
\begin{equation}
q_{[A,B],h}(x;S)
=
\frac{q_h(x;S)\mathbf{1}_{[A,B]}(x)}
{\int_A^B q_h(u;S)\,du}.
\end{equation}
Thus both the target density and the KDE are rewritten as probability densities on the same bounded interval. The corresponding bounded logarithmic score criterion is
\begin{equation}
\mathcal{L}_{[A,B]}(h;S,f)
=
\mathbb{E}_{X\sim f_{[A,B]}}
\left[
-\log_2 q_{[A,B],h}(X;S)
\right].
\end{equation}
In applications, the interval $[A,B]$ is treated as a working domain rather than as an
estimate of the full mathematical support. When physical or operational
bounds are available, they provide a natural choice. Otherwise, $[A,B]$
may be specified from representative historical or reference data so as to
cover the range over which the density estimate is intended to be used, with
suitable margins to avoid truncating non-negligible probability mass. A data-driven alternative is to expand the observed sample range using the distribution-free rank identity that, for an i.i.d.\ continuous sample of size $n$, a new observation falls outside the current range with probability $2/(n+1)$. A reference model, such as the uniform distribution,
can then be used to translate this probability into an interval expansion.

To transfer the predictor between bounded intervals, a sample on $[A,B]$
is first mapped to the reference interval $[-1,1]$ by

\begin{equation}
T(x)=\frac{x-m}{a},
\qquad
m=\frac{A+B}{2},
\qquad
a=\frac{B-A}{2}.
\end{equation}

The predictor trained on $[-1,1]$ is then applied to $T(S)$, and the
resulting bandwidth is rescaled to the original interval:
\begin{equation}
h_{\theta,[A,B]}(S)
=
aF_\theta\!\left(\mathrm{feat}(T(S))\right).
\end{equation}
Proposition~1 shows that this bandwidth rescaling is exactly compatible with the bounded logarithmic score criterion.

\begin{proposition}\label{prop:affine-equivariance}
Let $S\in[A,B]^n$, let $T(S)\in[-1,1]^n$ be its standardized sample,
and let $\widetilde{f}$ denote the density of $Y=(X-m)/a$ for
$X\sim f_{[A,B]}$. If
$h_{[A,B]}=a\,h_{[-1,1]}$, then
\begin{equation}
\mathcal{L}_{[A,B]}
\bigl(h_{[A,B]};S,f\bigr)
=
\mathcal{L}_{[-1,1]}
\bigl(h_{[-1,1]};T(S),\widetilde{f}\bigr)
+\log_2 a .
\end{equation}
Consequently, the task-specific minimizing bandwidths on the two intervals
correspond under the scaling
$h_{[A,B]}=a\,h_{[-1,1]}$.
\end{proposition}

\begin{proof}
Under the change of variables $x=m+ay$, the target density, the KDE,
and its truncation normalizer transform as
\[
\begin{aligned}
f_{[A,B]}(x)
&=a^{-1}\widetilde{f}(y),\\
q_{[A,B],h_{[A,B]}}(x;S)
&=a^{-1}q_{[-1,1],h_{[-1,1]}}(y;T(S)),\\
Z_{[A,B],h_{[A,B]}}(S)
&=Z_{[-1,1],h_{[-1,1]}}(T(S)),
\end{aligned}
\]
where
$Z_{[A,B],h}(S)=\int_A^B q_h(u;S)\,du$
denotes the truncation normalizing constant. Substituting these into
$\mathcal{L}_{[A,B]}$, the integration Jacobian cancels the scale factor
of $f_{[A,B]}$, while the factor $a^{-1}$ inside the logarithm contributes
the additive constant $\log_2 a$. Since $\log_2 a$ is independent of the
bandwidth, the minimizing bandwidths rescale by the factor $a$.
\end{proof}

\section{Experimental Demonstration}

This section evaluates the proposed bandwidth-learning framework from three complementary perspectives. The first benchmark considers Gaussian sampling and examines finite-sample bandwidth behaviour in the simplest setting. The second considers a bounded multi-family task distribution and tests whether a single shared predictor remains effective across different distributional shapes. The third uses randomized finite Gaussian mixtures as the training task distribution and examines whether a single flexible density generator can provide a broad source of tasks for learning a reusable bandwidth rule.

Implementation choices are kept simple throughout. The aim is to evaluate the bandwidth-learning formulation itself rather than to optimize network architecture or numerical implementation. Accordingly, the predictors use low-dimensional sample features, with positivity of the bandwidth enforced through a softplus
output transformation. All logarithms use base 2, so logarithmic score differences have the usual interpretation in bits of information. A one-bit lower score corresponds to assigning, on average, twice as much probability density to the realized outcomes \citep{Du2021}.

\subsection{Comparison of Bandwidth Selectors under Gaussian Sampling}

This experiment compares the amortized bandwidth selector with three classical selectors under Gaussian sampling. This controlled setting removes variation in skewness, tail behaviour, and modality, allowing the comparison to isolate the effect of sample size on logarithmic score performance. Performance under non-Gaussian density shapes is examined in the subsequent experiments.

For this Gaussian diagnostic benchmark, the amortized selector is restricted to
the scale-equivariant form
\begin{equation}
    h_\theta(S)
    =
    \operatorname{sd}(S)F_\theta(n),
    \qquad
    F_\theta(n)>0.
\end{equation}
Thus, the network learns only a dimensionless bandwidth ratio as a function of sample size, while the underlying distributional shape is fixed. The amortized selector is compared with Silverman's rule, the Sheather--Jones selector, and least-squares cross-validation (LSCV). Their definitions and implementation details are provided in Section~A of the Supplementary Material \citep{LiangDu2026Supplement}. All four selectors are translation invariant and scale equivariant, up to negligible numerical implementation error. The corresponding derivations are given in Section~B of the same supplement. It is therefore sufficient to conduct the comparison under standard-normal sampling, since translating or rescaling a Gaussian distribution does not alter the ordering of the selectors or their
pairwise logarithmic score differences.

Training samples are generated from $\mathcal{N}(0,1)$, with
$n\sim\operatorname{Unif}\{5,\ldots,256\}$. Both the target density and the KDE are defined on the full real line; no truncation or renormalization is applied in this
experiment. The predictor is trained on online-generated Gaussian tasks, with model selection based on a fixed validation set. Details of the network architecture and training procedure are provided in Section~C of the Supplementary Material \citep{LiangDu2026Supplement}. For each displayed sample size, performance is evaluated on
$N_{\mathrm{test}}=30000$ independent test samples.

\begin{figure}[t]
\centering
\includegraphics[width=0.95\linewidth]
{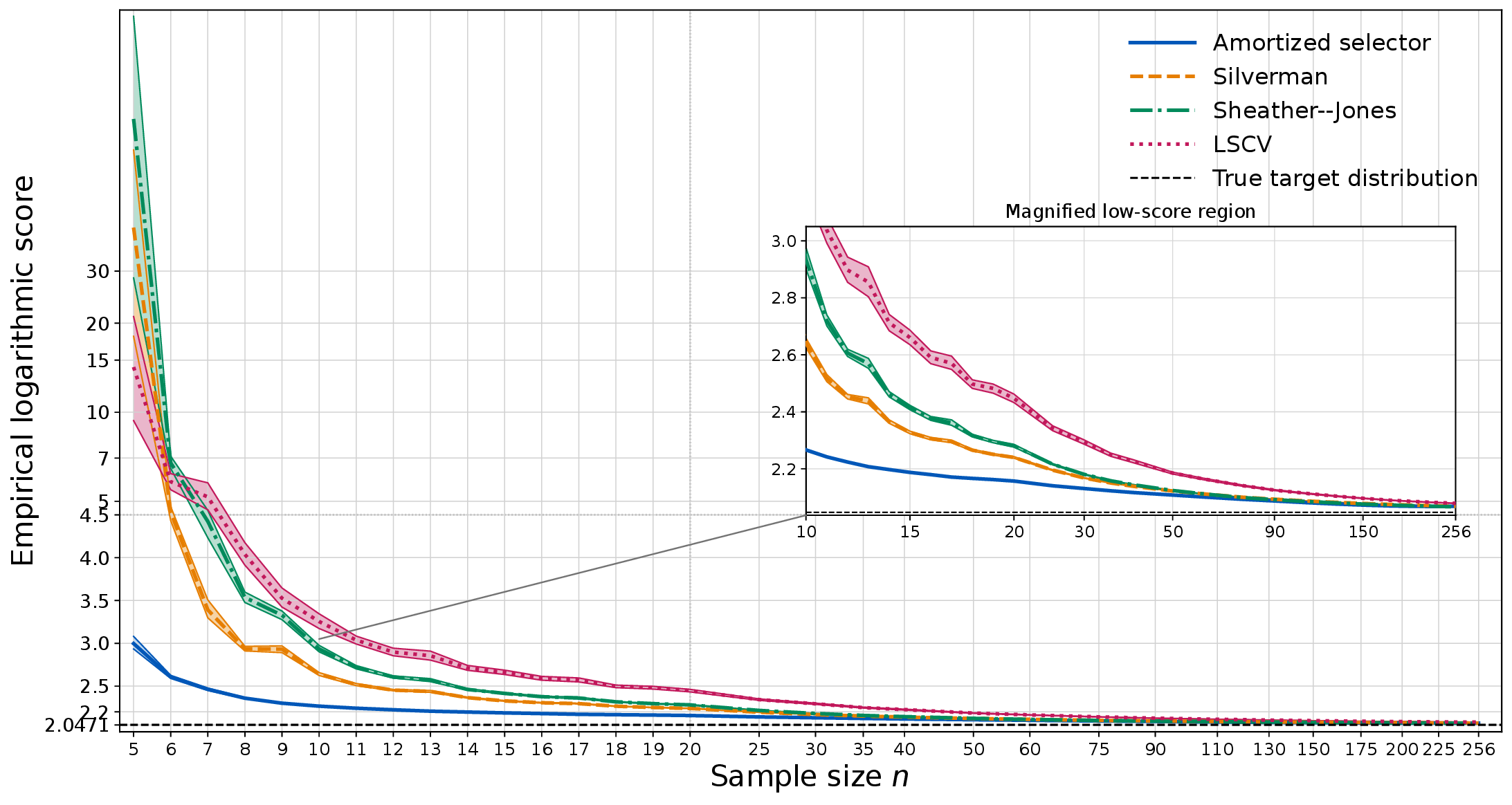}
\caption{Empirical logarithmic score versus sample size under Gaussian sampling. Scores are averaged over $N_{\mathrm{test}}=30000$ independent samples at each $n$; shaded bands show 90\% bootstrap intervals, and the black dashed line denotes the score of the true $\mathcal{N}(0,1)$ density. The horizontal axis is linear for $n\leq20$ and logarithmically compressed thereafter, while the vertical axis is linear up to $4.5$ bits and logarithmically compressed above.}
\label{fig:gaussian_comparison_empirical_score}
\end{figure}

Figure~\ref{fig:gaussian_comparison_empirical_score} shows that the amortized selector attains the lowest empirical logarithmic score throughout the displayed sample-size range and remains consistently closer to the standard-normal reference than the three classical selectors. The poor performance of the classical selectors in very small samples is consistent with occasional severe undersmoothing. Accidental clustering of observations can lead to an extremely small selected bandwidth, causing the resulting KDE to assign very low density to outcomes that
remain plausible under the standard-normal distribution and hence to incur large logarithmic penalties. The amortized selector is less sensitive to accidental local spacing
irregularities because its bandwidth ratio is learned across training tasks rather than determined by sample-specific spacing patterns. Training under the logarithmic score also discourages bandwidths that assign negligible density to outcomes that remain plausible under the data-generating distribution. This strong penalization is not an artefact of the logarithmic score, but an appropriate response to severe density-estimation errors \citep{Du2021}. As the sample size increases, all four selectors move towards the true standard-normal reference and their differences gradually narrow, although the amortized selector
retains a clear advantage through at least $n=50$.

\subsection{Comparison of Bandwidth Selectors under Bounded
Multi-Family Sampling}
\label{sec:multifamily}

The Gaussian comparison isolates finite-sample bandwidth behaviour under
a single regular density family. The present experiment evaluates
whether a shared amortized selector remains effective across a much
broader collection of density shapes, including skewed, heavy-tailed,
boundary-concentrated, multimodal, and multi-scale distributions.

The benchmark comprises ten equally weighted distribution families: Gaussian,
Laplace, Student-$t$, Gamma, Beta, Logistic, Lognormal, Bimodal, Trimodal,
and Spike--slab. Family-specific parameter ranges are chosen so that at least
$90\%$ of the untruncated probability mass lies in $[-1,1]$, providing a
practical balance between concentrating most probability mass within the working interval and retaining some tail mass outside it. Additional checks with alternative thresholds, not reported here, gave qualitatively similar results. Further details are given in Table~\ref{tab:multifamily_generators}. The bounded-support formulation of Section~2.3 is then applied.

\begin{table}[t]
\caption[Bounded multi-family benchmark]
{Densities and parameter specifications for the bounded multi-family benchmark.}
\label{tab:multifamily_generators}
\centering
\small
\setlength{\tabcolsep}{0pt}
\renewcommand{\arraystretch}{0.98}
\setlength{\jot}{0pt}
\begin{tabular}{
@{}
>{\raggedright\arraybackslash}p{0.12\linewidth}
@{\hspace{7pt}}
>{\raggedright\arraybackslash}p{0.34\linewidth}
@{\hspace{7pt}}
>{\raggedright\arraybackslash}p{\dimexpr0.54\linewidth-14pt\relax}
@{}
}
\toprule
Family
&
Density \(g(x)\)
&
Parameters
\\
\midrule
Gaussian
&
\(g(x)=\phi_{\sigma}(x-\mu)\)
&
\(\sigma\in[0.13,0.61]\).
\\[1pt]

Laplace
&
\(g(x)=(2b)^{-1}\exp\{-|x-\mu|/b\}\)
&
\(b\in[0.06,0.43]\).
\\[1pt]

Student-\(t\)
&
\(g(x)=s^{-1}t_{\nu}\{(x-\mu)/s\}\)
&
\(\nu\sim LU[3,30]\);
\(s\in[0.02,0.59]\).
\\[1pt]

Gamma
&
\(
\begin{array}{@{}l@{}}
y=k\theta+\varepsilon(x-\mu),\\[-2pt]
g(x)=\gamma_{k,\theta}(y)
\mathbf{1}_{\{y>0\}} .
\end{array}
\)
&
\(k\sim LU[0.5,10]\);
\(\varepsilon\in\{-1,1\}\) equiprobably;
\(\theta\in[0.04,1.4]\).
\\[1pt]

Beta
&
\(
\begin{array}{@{}l@{}}
z=\dfrac{a}{a+b}+\dfrac{x-\mu}{s},\\[-2pt]
g(x)=s^{-1}\beta_{a,b}(z)
\mathbf{1}_{\{0<z<1\}} .
\end{array}
\)
&
Four equiprobable regimes:
\(a,b\stackrel{\mathrm{iid}}{\sim}U[0.4,0.9]\),
\(U[1.2,5]\), or \(U[0.8,1.2]\);
or \(a\) and \(b\) are drawn, in random order, from
\(U[0.4,0.9]\) and \(U[1.2,5]\).
\(s\in[0.7,8.3]\).
\\[1pt]

Logistic
&
\(
g(x)=
\frac{\exp\{-(x-\mu)/s\}}
{s[1+\exp\{-(x-\mu)/s\}]^2}
\)
&
\(s\in[0.06,0.34]\).
\\[1pt]

Lognormal
&
\(
\begin{array}{@{}l@{}}
y=e^{\tau^2/2}+\varepsilon(x-\mu)/s,\\[-2pt]
g(x)=s^{-1}\ell_{\tau}(y)
\mathbf{1}_{\{y>0\}} .
\end{array}
\)
&
\(\tau\sim U[0.25,1]\);
\(\varepsilon\in\{-1,1\}\) equiprobably;
\(s\in[0.01,2.4]\).
\\[2pt]

Bimodal
&
\(
g(x)=
w\phi_{\sigma_1}(x-\mu_1)
+
(1-w)\phi_{\sigma_2}(x-\mu_2)
\)
&
\(w\sim U[0.2,0.8]\).
\\[-1pt]

&
\(
\begin{array}{@{}l@{}}
(m_1,m_2)=(-1,1),\\[-1pt]
\bar m=wm_1+(1-w)m_2
\end{array}
\)
&
\(\rho_1,\rho_2
\stackrel{\mathrm{iid}}{\sim}
LU[0.15,0.45]\).
\\[-1pt]

&
\(
\begin{array}{@{}l@{}}
\mu_j=\mu+\lambda(m_j-\bar m),\\[-1pt]
\sigma_j=\lambda\rho_j
\end{array}
\)
&
\(\mu_1\in[-1.0,0.4]\),
\(\mu_2\in[-0.4,1.0]\);
\(\sigma_1,\sigma_2\in[0.04,0.34]\).
\\[3pt]

Trimodal
&
\(
g(x)=
\sum_{j=1}^{3}
w_j\phi_{\sigma_j}(x-\mu_j)
\)
&
\(
(w_1,w_2,w_3)
\sim\operatorname{Dirichlet}(2,2,2)
\),
\(\min_j w_j\geq0.1\).
\\[-1pt]

&
\(
(m_1,m_2,m_3)=(-d_L,0,d_R)
\)
&
\(d_L,d_R
\stackrel{\mathrm{iid}}{\sim}
U[0.8,1.2]\).
\\[-1pt]

&
\(
\bar m=\sum_{j=1}^{3}w_jm_j
\)
&
\(\rho_j
\stackrel{\mathrm{iid}}{\sim}
LU[0.1,0.3]\).
\\[-1pt]

&
\(
\begin{array}{@{}l@{}}
\mu_j=\mu+\lambda(m_j-\bar m),\\[-1pt]
\sigma_j=\lambda\rho_j
\end{array}
\)
&
\(\mu_1\in[-1.3,0.4]\),
\(\mu_2\in[-0.7,0.7]\),
\(\mu_3\in[-0.4,1.4]\);
\(\sigma_1\in[0.03,0.29]\),
\(\sigma_2\in[0.03,0.31]\),
\(\sigma_3\in[0.03,0.28]\).
\\[3pt]

Spike--slab
&
\(
\begin{array}{@{}l@{}}
g(x)=
w\phi_{\sigma_{\mathrm{spike}}}(x-\mu)\\[-2pt]
\qquad+
(1-w)\phi_{\sigma_{\mathrm{slab}}}(x-\mu)
\end{array}
\)
&
\(w\sim U[0.2,0.8]\).
\\[-1pt]

&
\(
(\sigma_{\mathrm{spike}},
 \sigma_{\mathrm{slab}})
=
\lambda(\rho,1)
\)
&
\(\rho\sim LU[0.05,0.20]\);
\(\sigma_{\mathrm{spike}}\in[0.01,0.26]\);
\(\sigma_{\mathrm{slab}}\in[0.12,1.43]\).
\\
\bottomrule
\end{tabular}

\vspace{3pt}
\parbox{\linewidth}{\small
\textit{Note:}
Each family contributes \(10{,}000\) independently generated settings
(\(100{,}000\) in total).
For each setting, independent
\(\mu\sim U[-0.5,0.5]\) and \(c\sim U[0.9,1]\) are drawn.
Given the family-specific shape parameters, the overall scale
(\(\lambda\) for mixtures) is chosen numerically to satisfy
\(\int_{-1}^{1}g(u)\,du=c\).
Intervals without an explicit sampling distribution are realised ranges
over the \(10{,}000\) settings for that family.
Sampling and evaluation use
\(f_{[-1,1]}(x)=
g(x)\mathbf{1}_{[-1,1]}(x)/\int_{-1}^{1}g(u)\,du\).
Here, \(\phi_{\sigma}\), \(t_{\nu}\), \(\gamma_{k,\theta}\),
\(\beta_{a,b}\), and \(\ell_{\tau}\) denote the Gaussian,
standard Student-\(t\), Gamma, Beta, and lognormal densities;
\(U[a,b]\equiv\operatorname{Unif}[a,b]\) and
\(LU[a,b]\equiv\operatorname{LogUnif}[a,b]\).
Displayed realised bounds are rounded to one or two decimal places;
all computations use full-precision values.
}
\end{table}
For each family, 10,000 parameter settings are generated in advance and sampled during training, with
$n\sim\operatorname{Unif}\{5,\ldots,256\}$. The amortized selector uses the feature representation in~(4); further implementation details are provided in Section~C of the Supplementary Material \citep{LiangDu2026Supplement}. All comparisons are conducted on $[-1,1]$, since Proposition~1 together
with the scale equivariance of the selectors implies that the results transfer directly to any affine image of this interval. At each displayed sample size, performance is averaged equally across the ten families using 3000 independently generated test tasks per family.

\begin{figure}[t]
\centering
\includegraphics[width=0.95\linewidth]
{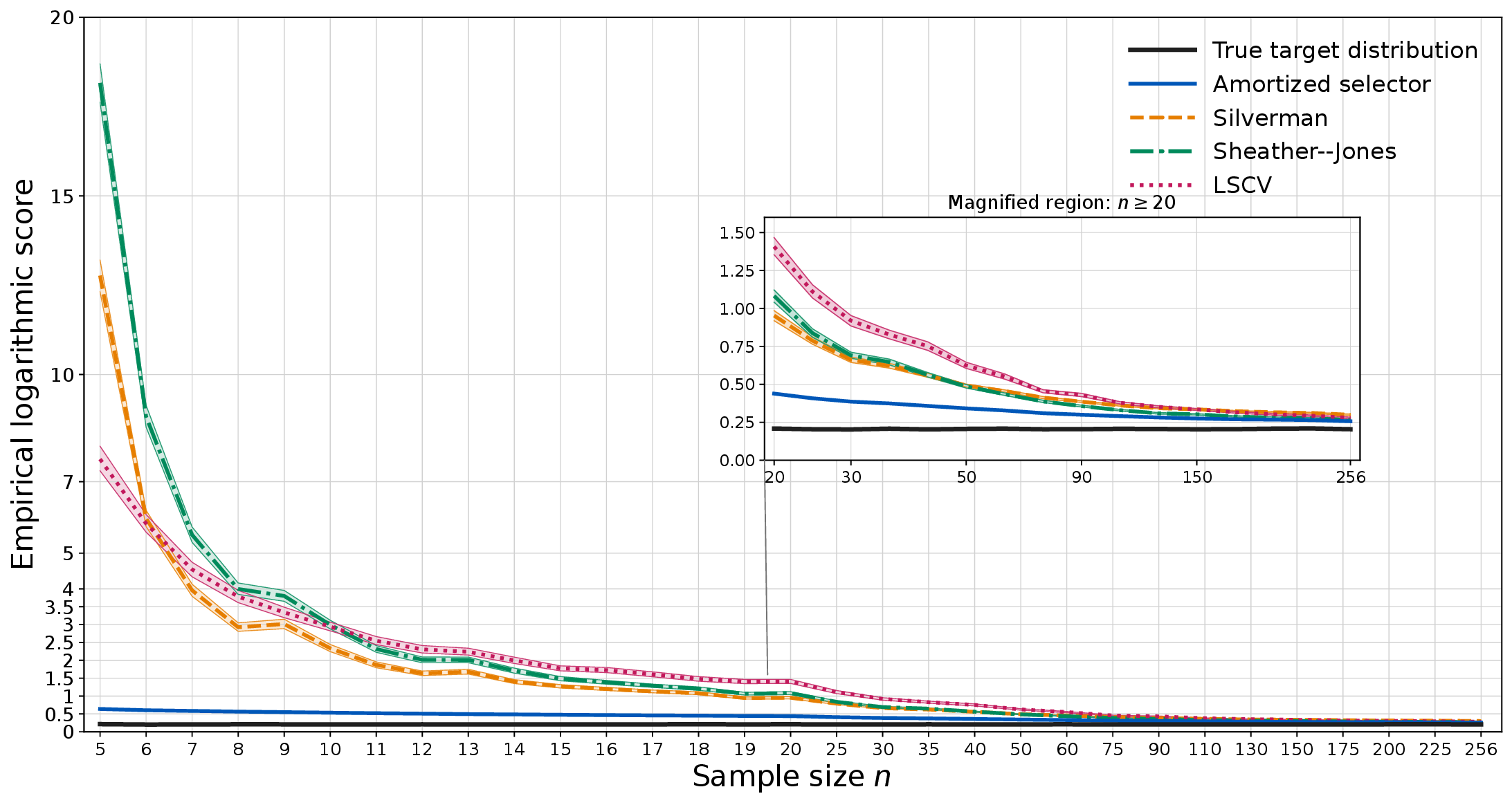}
\caption{Empirical logarithmic score versus sample size under bounded multi-family sampling on $[-1,1]$. Scores are averaged equally across the ten families using $N_{\mathrm{test}}=3000$ independent tasks per family at each sample size. Shaded bands show 90\% bootstrap intervals, and the black line denotes the corresponding score of the true target densities.}

\label{fig:multifamily_comparison_empirical_score}
\end{figure}

Figure~\ref{fig:multifamily_comparison_empirical_score} shows that the amortized selector attains the lowest empirical logarithmic score throughout the displayed sample-size range and remains consistently closer to the true-density reference than the three classical selectors. Its advantage is more pronounced than under Gaussian sampling, indicating that the improvement extends to heterogeneous density shapes. The stronger separation is consistent with the fact that the classical selectors are neither adapted to this task distribution nor optimized
directly under the logarithmic score.

\subsection{Comparison of Bandwidth Selectors under Bounded GMM Sampling}
\label{sec:gmm}

The multi-family benchmark shows that a single amortized
selector can remain effective across a broad range of density families.
Its construction, however, requires an explicit choice of which
distribution families to include. This motivates examining whether a
single flexible density generator can provide a similarly broad
training distribution.

Gaussian mixture models provide a natural candidate for this purpose.
A finite Gaussian mixture has the form
\begin{equation}
    g_K(x)
    =
    \sum_{k=1}^{K}
    w_k\phi_{\sigma_k}(x-\mu_k),
    \qquad
    w_k\geq0,
    \qquad
    \sum_{k=1}^{K}w_k=1.
\end{equation}
\citet{NguyenEtAl2023} show that finite location--scale mixtures
generated from an essentially bounded base density can approximate
target densities arbitrarily well in $L^p$. Since the Gaussian density
satisfies this condition, finite Gaussian mixtures can approximate
arbitrary probability densities in $L^1$ as the number and parameters
of the components vary. This provides a theoretical motivation for
using GMMs as a flexible source of density-estimation tasks. This
approximation property is preserved under truncation and
renormalization on $[-1,1]$.

Randomized finite Gaussian mixtures are generated for three component
counts, $K\in\{8,16,32\}$. For each fixed $K$, the mixture weights and
component locations are generated according to
\begin{equation}
\begin{aligned}
    \alpha
    &\sim
    \operatorname{LogUnif}[0.5,4],
    \\
    (w_1,\ldots,w_K)
    &\sim
    \operatorname{Dirichlet}
    \left(
        \frac{\alpha}{K},
        \ldots,
        \frac{\alpha}{K}
    \right),
    \\
    \mu_k
    &\stackrel{\mathrm{iid}}{\sim}
    \operatorname{Unif}[-1,1].
\end{aligned}
\end{equation}
The use of $\alpha/K$ keeps the total Dirichlet concentration equal to
$\alpha$ for all $K$, so that changing $K$ does not simultaneously alter
the overall concentration of the mixture weights. This scaling also has
a theoretical motivation: finite mixtures with symmetric Dirichlet
weights of this form approach Dirichlet-process mixtures with
concentration parameter $\alpha$ as $K$ increases
\citep{GreenRichardson2001}. Thus, $K$ controls the number of available
mixture components, whereas $\alpha$ controls the degree of weight
imbalance across tasks. The range $\alpha\in[0.5,4]$, sampled on the
logarithmic scale, is chosen to generate substantial variation in
mixture-weight concentration rather than to represent theoretically
optimal values. The component locations are sampled independently over
the reference interval without imposing separation or regular-spacing
constraints, allowing components to overlap naturally so that the
nominal component count $K$ does not directly determine the number of
modes.

Component scales are generated in two stages. A task-level reference
scale and a within-task log-scale dispersion are first sampled as
\begin{equation}
    \sigma_0
    \sim
    \operatorname{LogUnif}[0.025,0.22],
    \qquad
    \tau
    \sim
    \operatorname{Unif}[0.10,1.00].
\end{equation}
Conditional on $(\sigma_0,\tau)$, the component log-scales are sampled
according to
\begin{equation}
    \log\sigma_k
    \stackrel{\mathrm{iid}}{\sim}
    \mathcal{N}
    \left(
        \log\sigma_0,\tau^2
    \right),
\end{equation}
with sampling restricted to $\sigma_k\in[0.015,0.45]$. The reference
scale $\sigma_0$ determines the typical component width within a task,
whereas $\tau$ controls the heterogeneity of component widths.
Together, they allow both overall and local scales to vary across
tasks. The bounds on $\sigma_k$ exclude extremely narrow or diffuse
components relative to the reference interval $[-1,1]$. These ranges
are chosen to provide broad but nondegenerate variation. The
bounded-support formulation of Section~2.3 is then applied to each
generated mixture on $[-1,1]$.

The GMM construction used here is not tuned to the test distributions
and is not intended to be optimal. Additional checks under reasonable variations of the GMM construction, not reported here, gave qualitatively similar results. Other common GMM constructions are designed primarily for mixture inference rather than for generating a heterogeneous collection of density-estimation tasks. Bayesian mixture models, for example, often introduce hierarchical or shrinkage priors to regularize component parameters or to encourage a smaller
number of occupied components \citep{RichardsonGreen1997,MalsinerWalliEtAl2016}. Such constructions are well suited to mixture inference but impose
additional structure on component weights or scales, while equal-weight mixtures eliminate variation in the relative importance of the components. The present construction instead uses direct randomization to allow weight imbalance, component overlap, overall scale, and within-task scale heterogeneity to vary across tasks.

Because overlapping components can produce fewer modes than the nominal component count $K$, a Monte Carlo diagnostic is conducted using $5000$ independently generated densities for each value of $K$. Table~\ref{tab:gmm_mode_distribution} reports the empirical distribution of the detected number of modes. The realized modal complexity changes only moderately with $K$.
Across all three generators, two or three modes are most common, and most densities contain no more than four modes. Even for $K=32$, only $6.0\%$ of densities have six or more modes. Thus, increasing $K$ adds available mixture components without producing a proportional increase in modal complexity.

\begin{table}[t]
\caption[Detected mode counts under the bounded GMM generators]
{Empirical distributions of detected mode counts under the bounded GMM
generators. Entries are percentages from \(5000\) independently generated
densities for each \(K\).}
\label{tab:gmm_mode_distribution}
\centering
\small
\setlength{\tabcolsep}{3pt}
\renewcommand{\arraystretch}{1.15}
\begin{tabular*}{\linewidth}{@{\extracolsep{\fill}}lrrrrrrrrr@{}}
\toprule
& \multicolumn{9}{c}{Number of detected modes} \\
\cmidrule(lr){2-10}
Number of components \(K\)
& 1 & 2 & 3 & 4 & 5 & 6 & 7 & 8 & \(\geq 9\) \\
\midrule
8  & 17.0\% & 36.2\% & 27.2\% & 13.7\% & 4.8\% & 1.0\% & 0.2\% & 0.0\% & 0.0\% \\
16 & 14.4\% & 32.0\% & 26.8\% & 15.2\% & 7.3\% & 3.1\% & 0.9\% & 0.2\% & 0.0\% \\
32 & 13.5\% & 30.0\% & 26.3\% & 15.5\% & 8.7\% & 4.0\% & 1.2\% & 0.6\% & 0.2\% \\
\bottomrule
\end{tabular*}
\par\smallskip
\parbox{\linewidth}{\small\textit{Note:}
Modes are detected numerically using a \(2\%\) relative-prominence
threshold; boundary modes are included.}
\end{table}

Three amortized selectors are then trained separately using the
$K=8$, $K=16$, and $K=32$ GMM generators, with a fresh GMM generated independently for each training task. Sample sizes are drawn uniformly from $\{5,\ldots,256\}$, and the selector uses the same feature representation, bounded logarithmic score objective, network architecture, and model-selection procedure as in the multi-family experiment. Implementation details are provided in Section~C of the Supplementary Material \citep{LiangDu2026Supplement}.

Evaluation follows the protocol of Section~\ref{sec:multifamily}.
The three GMM-trained selectors are tested on independently generated bounded $K=32$ GMM tasks and on the same bounded multi-family benchmark. The previously trained multi-family amortized selector is included in both comparisons.

\begin{figure}[t]
\centering
\includegraphics[width=0.95\linewidth]
{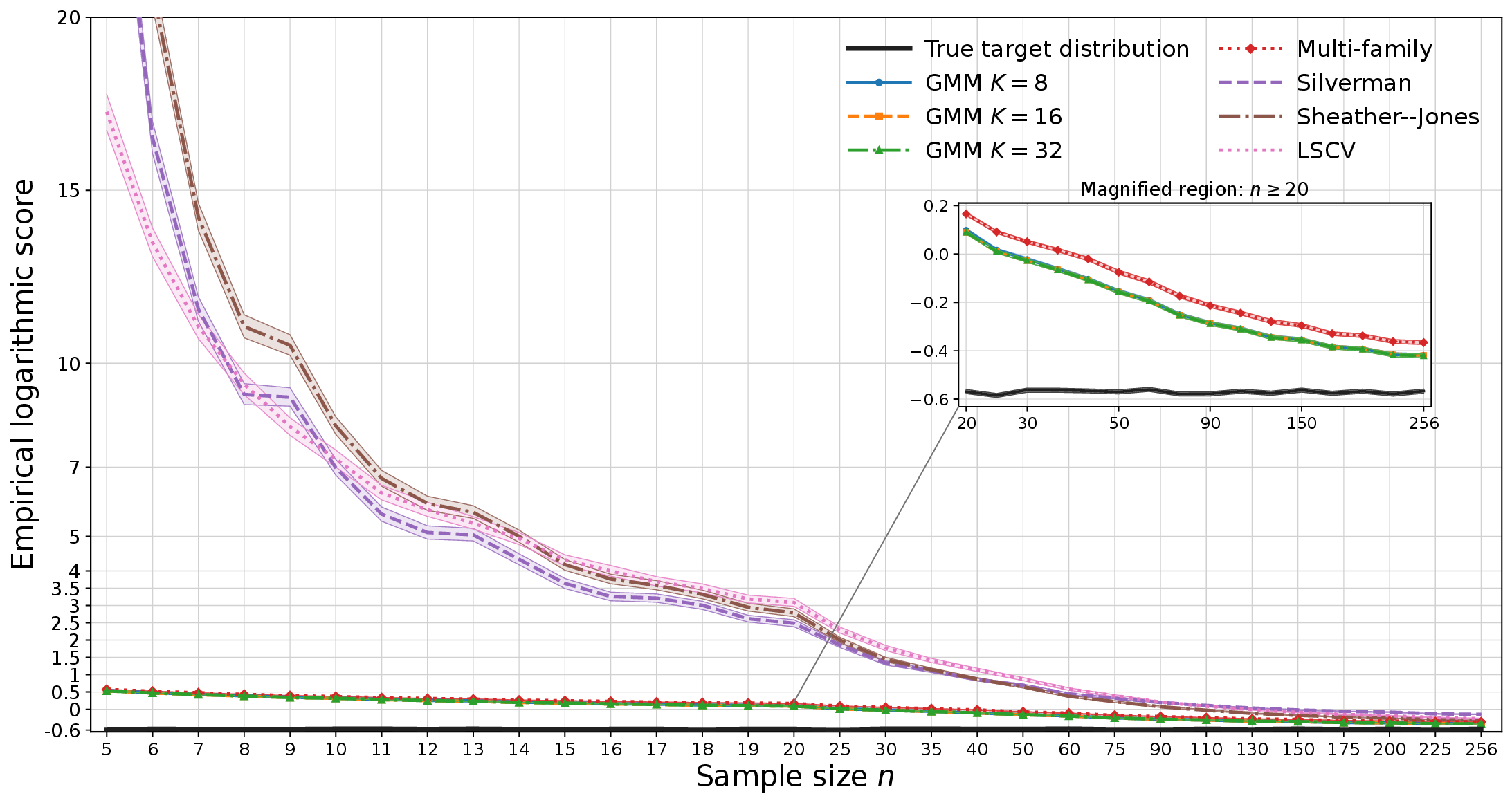}
\caption{Empirical logarithmic score versus sample size under bounded $K=32$ GMM sampling on $[-1,1]$. Scores are averaged over
$N_{\mathrm{test}}=30000$ independent GMM tasks at each sample size. GMM $K=8$, GMM $K=16$, and GMM $K=32$ denote amortized selectors trained using the respective GMM generators, while Multi-family denotes the selector trained on the bounded multi-family task distribution. Shaded bands show $90\%$ bootstrap intervals, and the black line denotes the corresponding score of the true target densities.}

\label{fig:gmm_k32_comparison_empirical_score}
\end{figure}

Figure~\ref{fig:gmm_k32_comparison_empirical_score} shows a more
pronounced separation between the GMM-trained amortized selectors and the classical selectors than in the bounded multi-family experiment, particularly at small and moderate sample sizes. All three GMM-trained selectors remain substantially closer to the true-density reference throughout the displayed sample-size range, with the gap narrowing as $n$ increases. The multi-family-trained selector also performs well on the GMM test
distribution, although its empirical logarithmic score lies slightly above the GMM-trained selectors. This provides further evidence of cross-distribution generalization. As an additional check on the role of truncation at the evaluation stage,
the learned bandwidths were also used to construct ordinary Gaussian KDEs
on $\mathbb{R}$ and evaluated against the corresponding untruncated target
densities; the relative performance of the selectors was qualitatively
unchanged (results not shown).

\begin{figure}[t]
\centering
\includegraphics[width=0.95\linewidth]
{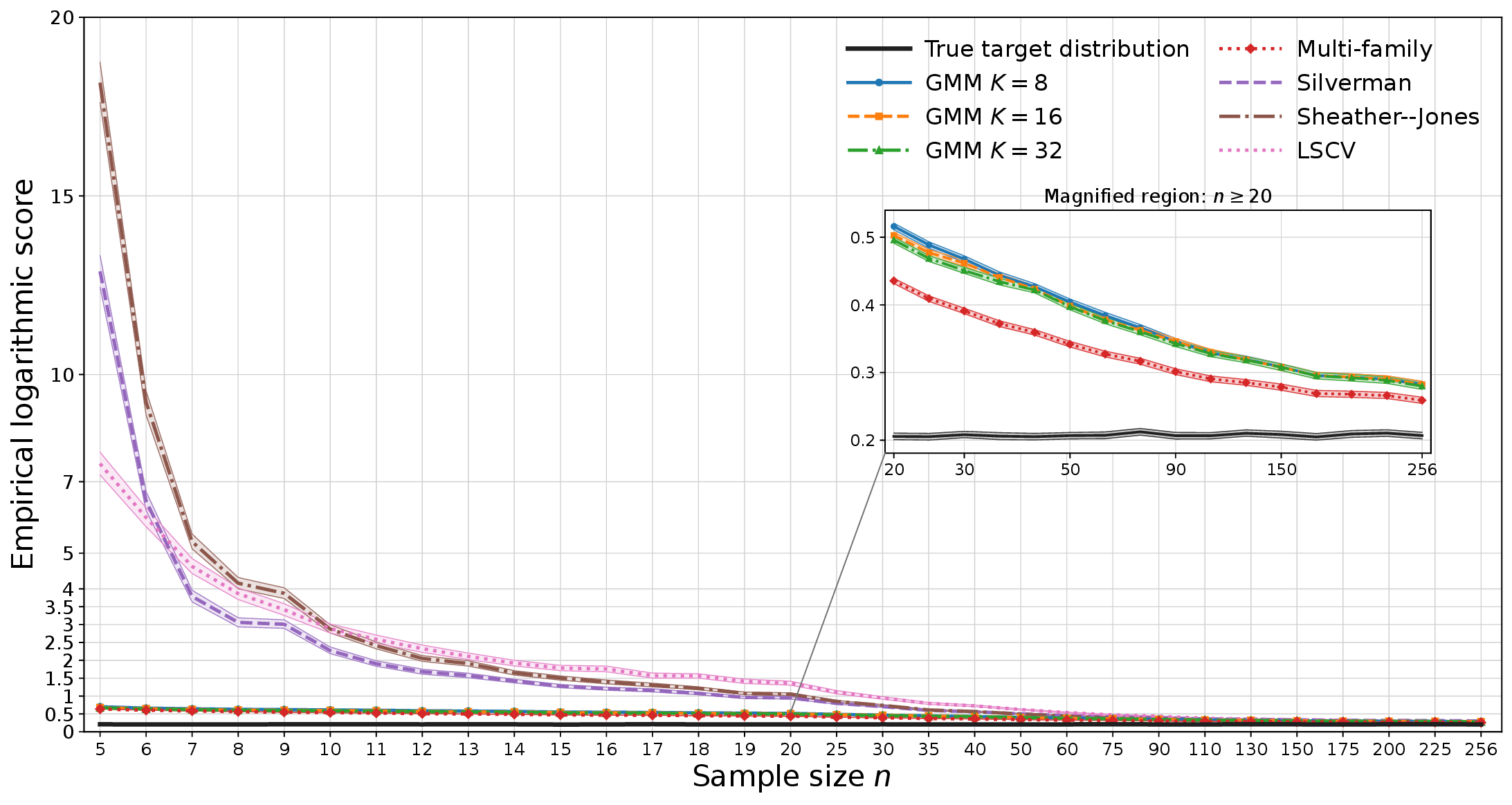}
\caption{Empirical logarithmic score versus sample size under bounded
multi-family sampling on $[-1,1]$. Scores are averaged equally across
the ten families using $N_{\mathrm{test}}=3000$ independent tasks per
family at each sample size. GMM $K=8$, $16$, and $32$ denote selectors
trained using the corresponding GMM generators, while Multi-family
denotes the multi-family-trained selector. Shaded bands show $90\%$
bootstrap intervals, and the black line denotes the corresponding score
of the true target densities.}

\label{fig:gmm_multifamily_comparison_empirical_score}
\end{figure}

Figure~\ref{fig:gmm_multifamily_comparison_empirical_score} shows that the selector trained directly on the multi-family task distribution attains the lowest empirical logarithmic score among the learned selectors on its matched test distribution. Nevertheless, all three GMM-trained selectors remain close to the multi-family-trained selector and substantially outperform the classical bandwidth selectors, particularly at small and moderate sample sizes. These results show that the GMM-trained selectors generalize effectively to the multi-family
benchmark, even though the ten distribution families are not explicitly included in their training task distributions.

Across both the GMM and multi-family test distributions, the three GMM-trained selectors produce very similar empirical logarithmic scores, indicating that performance is not strongly sensitive to the component count over the range considered. The mode-count diagnostic in Table~\ref{tab:gmm_mode_distribution} provides complementary evidence
that the realized modal complexity changes only moderately as $K$
increases.

Together, the $L^1$ approximation property and the empirical
cross-distribution results support randomized finite GMMs as a flexible training distribution for amortized bandwidth learning.

\section{Summary and Discussion}

This paper develops an amortized approach to bandwidth selection for
Gaussian-kernel density estimation. Classical bandwidth selectors obtain
the bandwidth through analytically or asymptotically prescribed rules or
through repeated sample-specific optimization, as represented here by
Silverman's rule, the Sheather--Jones plug-in selector, and least-squares
cross-validation. In contrast, the proposed framework learns the
sample-to-bandwidth mapping itself across a distribution of
density-estimation tasks. The density estimator remains classical KDE;
only the bandwidth-selection rule is learned. Training is carried out
directly under the logarithmic score, aligning bandwidth selection with
the probabilistic criterion used to evaluate the resulting density
estimate.

A key part of the framework is the bounded-support formulation. Without
restricting training to a common region of interest, rare realizations in
remote low-density regions can dominate the logarithmic score and drive
the learned bandwidth towards excessive smoothing, making it difficult to
train a single selector effectively across heterogeneous tasks. The
truncated-and-renormalized formulation addresses this problem by placing
the target density and the KDE on a common bounded interval and
concentrating the learning objective on the region of practical interest.
This restriction is not particularly limiting in practice. The bounded
interval is treated as a working domain rather than as an estimate of the
full mathematical support. Physical or operational bounds provide a natural
choice when available; otherwise, the interval may be specified from
representative historical or reference data with suitable margins, or
constructed from a data-driven expansion of the observed sample range.
Proposition~1 further shows that affine standardization changes the bounded
logarithmic score only by an additive constant and rescales the corresponding
minimizing bandwidth, allowing a selector trained on the reference interval
to be transferred by standardization and bandwidth rescaling.

The empirical results show a consistent advantage of amortized bandwidth
selection over the classical selectors considered here. Under Gaussian
sampling, a setting particularly favourable to classical bandwidth
selection and especially to Silverman's normal-reference rule, the
amortized selector nevertheless attains the lowest empirical logarithmic
score throughout the examined sample-size range, with the largest gains
at small sample sizes. The multi-family experiment further strengthens this result. A single
amortized selector maintains strong performance across widely varying
density structures and shows an even larger advantage over the classical
selectors.

The GMM experiment further shows that a manually enumerated collection
of distribution families is not required to construct an effective training
task distribution. Finite Gaussian mixtures provide a flexible
density-approximation framework, supported by their $L^1$ approximation
property, and allow a broad range of density structures to be generated
within a single generative framework. The GMM-trained selectors show little
sensitivity to the component count over the range considered and substantially
outperform the classical selectors on independently generated GMM tasks.
More importantly, they remain close to the matched multi-family-trained
selector on the same benchmark, even though those distribution families are
not explicitly included in the GMM training distribution. Conversely, the
multi-family-trained selector also performs well on the GMM benchmark.
Once trained, the same amortized selector can be applied directly to finite
samples from unknown densities without specifying, identifying, or fitting
their distributional families. The cross-distribution experiments show that
this broad applicability is accompanied by strong empirical performance
across diverse density structures.

The significance of amortization is therefore not merely computational.
Once trained, the selector requires only a single forward evaluation for
each new sample, but the more fundamental change is statistical: information accumulated across a distribution of
density-estimation tasks is used to learn the bandwidth-selection
mapping itself. This
allows recurring relationships between observable sample characteristics
and appropriate smoothing to be exploited rather than rediscovered from
each finite sample in isolation. The especially large gains in small and
heterogeneous samples are consistent with this interpretation, since
these are precisely the settings in which an individual sample contains
limited information about the underlying density structure. The learned bandwidth
selector can therefore be viewed as a foundation-style statistical model for bandwidth
selection, trained across a broad distribution of density-estimation tasks and reused
directly on previously unseen density structures. An interactive implementation of the amortized bandwidth selector is
available through the
\href{https://amortized-kde.streamlit.app/}
{\textcolor{blue}{\underline{online application}}},
where users can apply the trained selector to one-dimensional samples and
access the source code and reproducibility materials.

The present study is restricted to one-dimensional Gaussian-kernel KDE
with a single global bandwidth and a small collection of interpretable
sample features. These choices make it possible to isolate the effect of
learning the bandwidth rule itself, but they also leave substantial scope
for extension. Richer feature representations could capture density
structure not represented by low-order moments, while alternative kernels,
locally adaptive bandwidths, and multivariate density estimation provide
natural extensions of the framework. The $L^1$ approximation property of
finite Gaussian mixtures further provides a principled basis for extending
the training distribution over increasingly broad classes of
density-estimation tasks without explicitly enumerating individual
distribution families. Of particular practical interest are finite
ensemble outputs, for which the underlying distributional shape and the
appropriate amount of smoothing can vary from one forecast instance to
another. In such settings, a reusable selector that adapts its bandwidth
to each observed ensemble while learning from a broad population of
density-estimation tasks provides a natural alternative to bandwidth rules
whose functional form is fixed in advance.

\begin{acks}[Acknowledgments]
Junyi Liang acknowledges support from the National Natural Science Foundation of China (NSFC; Grant No. 42450196).  Hailiang Du acknowledges support from the Durham University International Engagement Travel Grant, which helped establish the collaboration between Durham University and East China University of Science and Technology that underpins this work.
\end{acks}

\begin{supplement}
\stitle{Supplement to ``Amortized Bandwidth Learning for Kernel Density
Estimation under Logarithmic Score''}
\sdescription{Additional definitions and implementation details for the
bandwidth selectors, equivariance derivations, and details of the network
architecture and training procedure.}
\end{supplement}

\bibliographystyle{imsart-nameyear}
\bibliography{reference,AOS_supplement_reference}

\end{document}


\begin{frontmatter}

\title{Supplement to ``Amortized Bandwidth Learning for Kernel Density
Estimation under Logarithmic Score''}
\runtitle{Supplement to Amortized KDE Bandwidth Learning}

\begin{aug}

\author[A]{\fnms{Junyi}~\snm{Liang}
\ead[label=e1]{y20240092@mail.ecust.edu.cn}}

\author[B,A,C]{\fnms{Hailiang}~\snm{Du}\thanksref{t1}
\ead[label=e2]{hailiang.du@durham.ac.uk}}

\address[A]{School of Mathematics,
East China University of Science and Technology,
200237 Shanghai, China\printead[presep={,\ }]{e1}}

\address[B]{Department of Mathematical Sciences;
Institute of Hazard, Risk and Resilience,
Durham University,
DH1 3LE Durham, United Kingdom\printead[presep={,\ }]{e2}}

\address[C]{Data Science Institute;
Global School of Sustainability,
The London School of Economics and Political Science,
WC2A 2AE London, United Kingdom}

\thankstext{t1}{Corresponding author.}

\end{aug}

\end{frontmatter}

\renewcommand{\thesection}{\Alph{section}}
\numberwithin{equation}{section}

\section{Classical Bandwidth Selectors}
\label{app:classical_selectors}

This section describes the three classical bandwidth selectors used
throughout the experimental comparisons: Silverman's rule, the
Sheather--Jones selector, and least-squares cross-validation. The same
implementations are used in all experiments. In the bounded-support
experiments, each selector is applied to the observed sample on its
original scale, and the resulting bandwidth is subsequently used in
the truncated-and-renormalized KDE defined in Section~2.3 of the main paper. The corresponding location and scale transformation properties are established separately in Section~\ref{app:scale_equivariance}.

\subsection{Silverman's Rule}

Silverman's rule is a normal-reference bandwidth selector that combines
the usual sample standard deviation with a robust scale estimate based
on the interquartile range \citep{Silverman1986}. For an observed sample
\begin{equation}
    S=(x_1,\ldots,x_n),
\end{equation}
the bandwidth used in the experiments is
\begin{equation}
    h_{\mathrm{S}}(S)
    =
    0.9
    \min
    \left\{
        \operatorname{sd}(S),
        \frac{\operatorname{IQR}(S)}{1.34}
    \right\}
    n^{-1/5},
    \label{eq:appendix_silverman}
\end{equation}
where
\begin{equation}
    \operatorname{sd}(S)
    =
    \left[
        \frac{1}{n-1}
        \sum_{i=1}^{n}
        (x_i-\bar{x})^2
    \right]^{1/2}
\end{equation}
and
\begin{equation}
    \operatorname{IQR}(S)
    =
    Q_{0.75}(S)-Q_{0.25}(S).
\end{equation}
The minimum of the two scale estimates limits the influence of extreme
observations while retaining the familiar \(n^{-1/5}\) dependence of
the normal-reference bandwidth. The method is available in closed form
and requires no sample-specific numerical optimization.

\subsection{Sheather--Jones Selector}

The Sheather--Jones method is a data-based plug-in selector derived
from the asymptotic mean integrated squared error of the KDE
\citep{SheatherJones1991}. For a kernel \(\kappa\), define
\begin{equation}
    R(\kappa)
    =
    \int_{\mathbb{R}}\kappa(u)^2\,du,
    \qquad
    \mu_2(\kappa)
    =
    \int_{\mathbb{R}}u^2\kappa(u)\,du,
\end{equation}
and let
\begin{equation}
    R(f'')
    =
    \int_{\mathbb{R}}
    \{f''(x)\}^2\,dx
\end{equation}
denote the roughness of the unknown density. The leading asymptotic
criterion has the form
\begin{equation}
    \operatorname{AMISE}(h)
    =
    \frac{R(\kappa)}{nh}
    +
    \frac{h^4}{4}
    \mu_2(\kappa)^2
    R(f'').
\end{equation}
If \(R(f'')\) were known, minimizing this expression would give
\begin{equation}
    h_{\operatorname{AMISE}}
    =
    \left[
        \frac{R(\kappa)}
        {n\mu_2(\kappa)^2R(f'')}
    \right]^{1/5}.
    \label{eq:appendix_amise_bandwidth}
\end{equation}
For the standard Gaussian kernel,
\begin{equation}
    R(\kappa)=\frac{1}{2\sqrt{\pi}},
    \qquad
    \mu_2(\kappa)=1.
\end{equation}

The quantity \(R(f'')\) is unknown and must be estimated from the
sample. The Sheather--Jones construction uses pilot estimates of
density-derivative functionals and determines the final bandwidth
through a multistage plug-in procedure. In its solve-the-equation form,
the defining relation can be represented schematically as
\begin{equation}
    h
    =
    \left[
        \frac{R(\kappa)}
        {
            n\mu_2(\kappa)^2
            \widehat{R}_{p(h)}(f'')
        }
    \right]^{1/5},
    \qquad
    h>0,
    \label{eq:appendix_sj_equation}
\end{equation}
where \(p(h)\) is a pilot bandwidth and
\(\widehat{R}_{p(h)}(f'')\) is the associated estimate of the
roughness functional. Because the pilot estimate depends on the final
bandwidth through \(p(h)\), equation
\eqref{eq:appendix_sj_equation} must be solved numerically.

In the experiments, the Sheather--Jones bandwidth is computed using
\begin{equation}
    \texttt{bw.SJ(S, method = "ste")}
\end{equation}
in \textsf{R}, where \texttt{"ste"} denotes the solve-the-equation
version. Its selected positive solution is denoted by
\begin{equation}
    h_{\mathrm{SJ}}(S).
\end{equation}
The implementation uses pilot estimation of density derivatives and
numerical root finding. As with any numerical implementation, the
reported value may be affected by floating-point arithmetic,
discretization, and root-finding tolerance.

\subsection{Least-Squares Cross-Validation}

Least-squares cross-validation selects the bandwidth by minimizing a
sample-based estimate of integrated squared error
\citep{Rudemo1982,Bowman1984}. For a candidate bandwidth \(h>0\), let
\begin{equation}
    q_h(x;S)
    =
    \frac{1}{nh}
    \sum_{i=1}^{n}
    \phi
    \left(
        \frac{x-x_i}{h}
    \right)
\end{equation}
be the Gaussian KDE, and define the leave-one-out estimator
\begin{equation}
    q_{h,-i}(x;S)
    =
    \frac{1}{(n-1)h}
    \sum_{\substack{j=1\\j\ne i}}^{n}
    \phi
    \left(
        \frac{x-x_j}{h}
    \right).
\end{equation}
The least-squares cross-validation criterion is
\begin{equation}
    \operatorname{CV}(h;S)
    =
    \int_{\mathbb{R}}
    q_h(x;S)^2\,dx
    -
    \frac{2}{n}
    \sum_{i=1}^{n}
    q_{h,-i}(x_i;S).
    \label{eq:appendix_lscv}
\end{equation}
The term \(\int f(x)^2\,dx\) in the integrated squared error does not
depend on \(h\) and is therefore omitted from the criterion.

For the Gaussian kernel, both terms in
\eqref{eq:appendix_lscv} can be evaluated directly from pairwise sample
differences. Let
\begin{equation}
    \phi_{\sqrt{2}}(u)
    =
    \frac{1}{\sqrt{4\pi}}
    \exp\left(-\frac{u^2}{4}\right)
\end{equation}
denote the density of a centred Gaussian variable with variance \(2\).
Then
\begin{equation}
\begin{aligned}
    \operatorname{CV}(h;S)
    &=
    \frac{1}{n^2h}
    \sum_{i=1}^{n}
    \sum_{j=1}^{n}
    \phi_{\sqrt{2}}
    \left(
        \frac{x_i-x_j}{h}
    \right)
    \\
    &\quad
    -
    \frac{2}{n(n-1)h}
    \sum_{i=1}^{n}
    \sum_{\substack{j=1\\j\ne i}}^{n}
    \phi
    \left(
        \frac{x_i-x_j}{h}
    \right).
\end{aligned}
\label{eq:appendix_lscv_gaussian}
\end{equation}

Rather than minimizing the criterion over all \(h>0\), the experimental
implementation evaluates it over \(60\) candidate bandwidths
logarithmically spaced relative to Silverman's bandwidth. The
candidate grid is
\begin{equation}
    \mathcal{G}(S)
    =
    \left\{
        h_0(S),\ldots,h_{59}(S)
    \right\},
\end{equation}
where
\begin{equation}
    h_k(S)
    =
    0.2h_{\mathrm{S}}(S)
    \left(
        \frac{5}{0.2}
    \right)^{k/59},
    \qquad
    k=0,\ldots,59.
    \label{eq:appendix_lscv_grid}
\end{equation}
Thus,
\begin{equation}
    h_0(S)=0.2h_{\mathrm{S}}(S),
    \qquad
    h_{59}(S)=5h_{\mathrm{S}}(S),
\end{equation}
and adjacent grid points satisfy
\begin{equation}
    \frac{h_{k+1}(S)}{h_k(S)}
    =
    \left(
        \frac{5}{0.2}
    \right)^{1/59}.
\end{equation}
The grid therefore has constant multiplicative spacing rather than
constant additive spacing, providing comparable relative resolution
across small and large bandwidths.

The selected LSCV bandwidth is
\begin{equation}
    h_{\mathrm{LSCV}}(S)
    \in
    \operatorname*{arg\,min}_{h\in\mathcal{G}(S)}
    \operatorname{CV}(h;S).
    \label{eq:appendix_lscv_selector}
\end{equation}
Defining the search interval relative to \(h_{\mathrm{S}}(S)\) also
ensures that the complete candidate grid rescales with the observations,
which is used in the scale-equivariance argument in
Section~\ref{app:scale_equivariance}.

\subsection{Use in the Bounded-Support Experiments}

The definitions above determine the bandwidths before bounded-support
evaluation. For a sample observed on \([A,B]\), each classical selector
is applied directly to that sample on its original scale, producing
\begin{equation}
    h_b(S),
    \qquad
    b\in
    \{
        \mathrm{S},
        \mathrm{SJ},
        \mathrm{LSCV}
    \}.
\end{equation}
The resulting bandwidth is then inserted into the
truncated-and-renormalized estimator
\begin{equation}
    q_{[A,B],h_b(S)}(x;S)
    =
    \frac{
        q_{h_b(S)}(x;S)\mathbf{1}_{[A,B]}(x)
    }{
        \displaystyle
        \int_A^B q_{h_b(S)}(u;S)\,du
    }.
\end{equation}
Accordingly, all methods are evaluated using the same bounded-support
density representation and the same logarithmic score criterion. The
bounded-support transformation changes the density used for evaluation,
but it does not alter the bandwidth-selection algorithms specified
above.

\section{Scale Equivariance of the Gaussian-Comparison Bandwidth
Selectors}
\label{app:scale_equivariance}

This section establishes the location and scale transformation
properties used in the Gaussian comparison. Let
\begin{equation}
    S=(x_1,\ldots,x_n)
\end{equation}
be an observed sample, and define its affine transformation by
\begin{equation}
    S_{\delta,c}
    =
    \delta+cS
    =
    (\delta+cx_1,\ldots,\delta+cx_n),
    \qquad
    \delta\in\mathbb{R},
    \quad
    c>0.
\end{equation}
A bandwidth selector $h_b$ is translation invariant and scale
equivariant if
\begin{equation}
    h_b(S_{\delta,c})
    =
    c\,h_b(S).
\end{equation}
Thus, the bandwidth itself is not scale invariant: it changes in direct
proportion to the scale of the data. Dimensionless quantities such as
$h_b(S)/\operatorname{sd}(S)$ are scale invariant.

The derivations below concern the exact defining constructions of the
selectors. Numerical implementations may deviate from these exact
identities because of floating-point arithmetic, numerical root finding,
binning, or the treatment of tied minima.

\subsection{Affine Transformation of the Gaussian KDE and
Logarithmic Score}

For a bandwidth $h>0$, the Gaussian KDE based on $S$ is
\begin{equation}
    q_h(x;S)
    =
    \frac{1}{nh}
    \sum_{i=1}^n
    \phi
    \left(
        \frac{x-x_i}{h}
    \right),
\end{equation}
where $\phi$ is the standard-normal density. Consider the transformed
sample $S_{\delta,c}$ and the transformed bandwidth $ch$. For any
$y\in\mathbb{R}$,
\begin{equation}
\begin{aligned}
    q_{ch}(\delta+cy;S_{\delta,c})
    &=
    \frac{1}{nch}
    \sum_{i=1}^n
    \phi
    \left(
        \frac{\delta+cy-(\delta+cx_i)}{ch}
    \right)
    \\
    &=
    \frac{1}{nch}
    \sum_{i=1}^n
    \phi
    \left(
        \frac{y-x_i}{h}
    \right)
    \\
    &=
    \frac{1}{c}q_h(y;S).
\end{aligned}
\end{equation}
Hence,
\begin{equation}
    q_{ch}(\delta+cy;S_{\delta,c})
    =
    c^{-1}q_h(y;S).
    \label{eq:scale-equivariance-1}
\end{equation}

Let $f$ be a density on $\mathbb{R}$ and let
\begin{equation}
    f_{\delta,c}(x)
    =
    \frac{1}{c}
    f
    \left(
        \frac{x-\delta}{c}
    \right)
\end{equation}
be the density obtained by applying the transformation
$x\mapsto \delta+cx$. The logarithmic score criterion for the transformed
problem is
\begin{equation}
    \mathcal{L}
    \bigl(
        ch;S_{\delta,c},f_{\delta,c}
    \bigr)
    =
    -\int_{\mathbb{R}}
    \log_2 q_{ch}(x;S_{\delta,c})
    f_{\delta,c}(x)\,dx.
\end{equation}
Using the change of variables $x=\delta+cy$, so that $dx=c\,dy$, gives
\begin{equation}
\begin{aligned}
    \mathcal{L}
    \bigl(
        ch;S_{\delta,c},f_{\delta,c}
    \bigr)
    &=
    -\int_{\mathbb{R}}
    \log_2
    q_{ch}(\delta+cy;S_{\delta,c})
    f(y)\,dy
    \\
    &=
    -\int_{\mathbb{R}}
    \log_2
    \left[
        \frac{1}{c}q_h(y;S)
    \right]
    f(y)\,dy
    \\
    &=
    -\int_{\mathbb{R}}
    \log_2 q_h(y;S)f(y)\,dy
    +
    \log_2 c
    \\
    &=
    \mathcal{L}(h;S,f)+\log_2 c.
\end{aligned}
\end{equation}
Therefore,
\begin{equation}
    \boxed{
    \mathcal{L}
    \bigl(
        ch;S_{\delta,c},f_{\delta,c}
    \bigr)
    =
    \mathcal{L}(h;S,f)+\log_2 c.
    }
    \label{eq:scale-equivariance-2}
\end{equation}

The additive term $\log_2 c$ does not depend on the bandwidth.
Consequently, the task-specific minimizing bandwidths transform as
\begin{equation}
    \operatorname*{arg\,min}_{h'>0}
    \mathcal{L}
    \bigl(
        h';S_{\delta,c},f_{\delta,c}
    \bigr)
    =
    c
    \operatorname*{arg\,min}_{h>0}
    \mathcal{L}(h;S,f).
    \label{eq:scale-equivariance-3}
\end{equation}

Moreover, if the two bandwidth selectors $h_{b_1}$ and $h_{b_2}$ are
scale equivariant, their logarithmic score difference is invariant under
the same affine transformation:
\begin{equation}
\begin{aligned}
    &
    \mathcal{L}
    \bigl(
        h_{b_1}(S_{\delta,c});
        S_{\delta,c},
        f_{\delta,c}
    \bigr)
    -
    \mathcal{L}
    \bigl(
        h_{b_2}(S_{\delta,c});
        S_{\delta,c},
        f_{\delta,c}
    \bigr)
    \\
    &\qquad=
    \mathcal{L}
    \bigl(
        h_{b_1}(S);S,f
    \bigr)
    -
    \mathcal{L}
    \bigl(
        h_{b_2}(S);S,f
    \bigr).
\end{aligned}
\label{eq:scale-equivariance-4}
\end{equation}

\subsection{Gaussian Amortized Selector}

In the Gaussian comparison, the amortized selector is parameterized as
\begin{equation}
    h_\theta(S)
    =
    \operatorname{sd}(S)F_\theta(n),
\end{equation}
where $F_\theta(n)>0$ is dimensionless and depends only on the sample
size.

The sample mean of $S_{\delta,c}$ is
\begin{equation}
    \overline{x}_{\delta,c}
    =
    \delta+c\overline{x}.
\end{equation}
Therefore,
\begin{equation}
\begin{aligned}
    \operatorname{sd}(S_{\delta,c})
    &=
    \left[
        \frac{1}{n-1}
        \sum_{i=1}^n
        \left\{
            (\delta+cx_i)-(\delta+c\overline{x})
        \right\}^2
    \right]^{1/2}
    \\
    &=
    \left[
        \frac{c^2}{n-1}
        \sum_{i=1}^n
        (x_i-\overline{x})^2
    \right]^{1/2}
    \\
    &=
    c\,\operatorname{sd}(S).
\end{aligned}
\end{equation}
Since the sample size is unchanged,
\begin{equation}
\begin{aligned}
    h_\theta(S_{\delta,c})
    &=
    \operatorname{sd}(S_{\delta,c})F_\theta(n)
    \\
    &=
    c\,\operatorname{sd}(S)F_\theta(n)
    \\
    &=
    c\,h_\theta(S).
\end{aligned}
\end{equation}
Hence,
\begin{equation}
    \boxed{
        h_\theta(\delta+cS)
        =
        c\,h_\theta(S).
    }
    \label{eq:scale-equivariance-5}
\end{equation}

\subsection{Silverman's Rule}

Silverman's bandwidth is
\begin{equation}
    h_{\mathrm{S}}(S)
    =
    0.9
    \min
    \left\{
        \operatorname{sd}(S),
        \frac{\operatorname{IQR}(S)}{1.34}
    \right\}
    n^{-1/5}.
\end{equation}

As shown above,
\begin{equation}
    \operatorname{sd}(S_{\delta,c})
    =
    c\,\operatorname{sd}(S).
    \label{eq:scale-equivariance-6}
\end{equation}
For $c>0$, empirical quantiles transform according to
\begin{equation}
    Q_p(S_{\delta,c})
    =
    \delta+cQ_p(S).
\end{equation}
It follows that
\begin{equation}
\begin{aligned}
    \operatorname{IQR}(S_{\delta,c})
    &=
    Q_{0.75}(S_{\delta,c})-Q_{0.25}(S_{\delta,c})
    \\
    &=
    \bigl[\delta+cQ_{0.75}(S)\bigr]
    -
    \bigl[\delta+cQ_{0.25}(S)\bigr]
    \\
    &=
    c\,\operatorname{IQR}(S).
\end{aligned}
\label{eq:scale-equivariance-7}
\end{equation}

Substitution into Silverman's formula gives
\begin{equation}
\begin{aligned}
    h_{\mathrm{S}}(S_{\delta,c})
    &=
    0.9
    \min
    \left\{
        c\,\operatorname{sd}(S),
        \frac{c\,\operatorname{IQR}(S)}{1.34}
    \right\}
    n^{-1/5}
    \\
    &=
    c\,
    0.9
    \min
    \left\{
        \operatorname{sd}(S),
        \frac{\operatorname{IQR}(S)}{1.34}
    \right\}
    n^{-1/5}
    \\
    &=
    c\,h_{\mathrm{S}}(S).
\end{aligned}
\end{equation}
Therefore,
\begin{equation}
    \boxed{
        h_{\mathrm{S}}(\delta+cS)
        =
        c\,h_{\mathrm{S}}(S).
    }
    \label{eq:scale-equivariance-8}
\end{equation}

\subsection{Sheather--Jones Selector}

The scale equivariance of the Sheather--Jones selector follows from
the homogeneity of its pilot derivative-functional estimates. The
solve-the-equation construction can be expressed in terms of estimates
of the form
\begin{equation}
    \widehat{\psi}_r(p;S)
    =
    \frac{1}{n^2p^{r+1}}
    \sum_{i=1}^n
    \sum_{j=1}^n
    L^{(r)}
    \left(
        \frac{x_i-x_j}{p}
    \right),
    \label{eq:scale-equivariance-9}
\end{equation}
where $L^{(r)}$ is the $r$th derivative of a pilot kernel and $p>0$
is a pilot bandwidth. Alternative normalizing conventions, such as
using $n(n-1)$ and excluding diagonal terms, do not affect the scaling
argument.

Because pairwise differences are unaffected by translation and scale
linearly with $c$,
\begin{equation}
    \frac{(\delta+cx_i)-(\delta+cx_j)}{cp}
    =
    \frac{x_i-x_j}{p}.
\end{equation}
Consequently,
\begin{equation}
\begin{aligned}
    \widehat{\psi}_r(cp;S_{\delta,c})
    &=
    \frac{1}{n^2(cp)^{r+1}}
    \sum_{i=1}^n
    \sum_{j=1}^n
    L^{(r)}
    \left(
        \frac{x_i-x_j}{p}
    \right)
    \\
    &=
    c^{-(r+1)}
    \widehat{\psi}_r(p;S).
\end{aligned}
\end{equation}
Thus,
\begin{equation}
    \boxed{
        \widehat{\psi}_r(cp;S_{\delta,c})
        =
        c^{-(r+1)}
        \widehat{\psi}_r(p;S).
    }
    \label{eq:scale-equivariance-10}
\end{equation}
In particular,
\begin{equation}
    \widehat{\psi}_4(cp;S_{\delta,c})
    =
    c^{-5}\widehat{\psi}_4(p;S),
    \label{eq:scale-equivariance-11}
\end{equation}
and
\begin{equation}
    \widehat{\psi}_6(cp;S_{\delta,c})
    =
    c^{-7}\widehat{\psi}_6(p;S).
    \label{eq:scale-equivariance-12}
\end{equation}

The solve-the-equation construction begins with a scale estimate
$s(S)$ satisfying
\begin{equation}
    s(S_{\delta,c})=c\,s(S).
    \label{eq:scale-equivariance-13}
\end{equation}
For example, a scale estimate based on the minimum of the sample
standard deviation and a rescaled interquartile range has this
property.

Let the initial pilot bandwidths be
\begin{equation}
    p_4(S)
    =
    C_4\,s(S)n^{-1/7},
    \qquad
    p_6(S)
    =
    C_6\,s(S)n^{-1/9},
    \label{eq:scale-equivariance-14}
\end{equation}
where $C_4$ and $C_6$ are dimensionless constants. Equation
\eqref{eq:scale-equivariance-13} implies
\begin{equation}
    p_4(S_{\delta,c})
    =
    c\,p_4(S),
    \qquad
    p_6(S_{\delta,c})
    =
    c\,p_6(S).
    \label{eq:scale-equivariance-15}
\end{equation}

Define the third-derivative roughness pilot estimate by
\begin{equation}
    \widehat{R}_3(S)
    =
    -\widehat{\psi}_6
    \bigl(
        p_6(S);S
    \bigr).
    \label{eq:scale-equivariance-16}
\end{equation}
Using \eqref{eq:scale-equivariance-12} and
\eqref{eq:scale-equivariance-15},
\begin{equation}
    \widehat{R}_3(S_{\delta,c})
    =
    c^{-7}\widehat{R}_3(S).
    \label{eq:scale-equivariance-17}
\end{equation}

The solve-the-equation pilot factor has the form
\begin{equation}
    \eta_{\mathrm{SJ}}(S)
    =
    C_\eta
    \left[
        \frac{
            \widehat{\psi}_4
            \bigl(
                p_4(S);S
            \bigr)
        }{
            \widehat{R}_3(S)
        }
    \right]^{1/7},
    \label{eq:scale-equivariance-18}
\end{equation}
where $C_\eta$ is dimensionless. From
\eqref{eq:scale-equivariance-11},
\eqref{eq:scale-equivariance-15}, and
\eqref{eq:scale-equivariance-17},
\begin{equation}
\begin{aligned}
    \eta_{\mathrm{SJ}}(S_{\delta,c})
    &=
    C_\eta
    \left[
        \frac{
            c^{-5}
            \widehat{\psi}_4
            \bigl(
                p_4(S);S
            \bigr)
        }{
            c^{-7}\widehat{R}_3(S)
        }
    \right]^{1/7}
    \\
    &=
    c^{2/7}\eta_{\mathrm{SJ}}(S).
\end{aligned}
\label{eq:scale-equivariance-19}
\end{equation}

For a candidate final bandwidth $h$, define the corresponding pilot
bandwidth by
\begin{equation}
    p_S(h)
    =
    \eta_{\mathrm{SJ}}(S)h^{5/7}.
    \label{eq:scale-equivariance-20}
\end{equation}
If $h$ is replaced by $ch$ and $S$ by $S_{\delta,c}$, then
\begin{equation}
\begin{aligned}
    p_{S_{\delta,c}}(ch)
    &=
    \eta_{\mathrm{SJ}}(S_{\delta,c})(ch)^{5/7}
    \\
    &=
    c^{2/7}\eta_{\mathrm{SJ}}(S)c^{5/7}h^{5/7}
    \\
    &=
    c\,p_S(h).
\end{aligned}
\label{eq:scale-equivariance-21}
\end{equation}

The Sheather--Jones bandwidth is defined as a positive solution of an
equation of the form
\begin{equation}
    h
    =
    \left[
        \frac{D_n}{
            \widehat{\psi}_4
            \bigl(
                p_S(h);S
            \bigr)
        }
    \right]^{1/5},
    \label{eq:scale-equivariance-22}
\end{equation}
where $D_n$ depends on the sample size and kernel constants but not on
the scale of the observations.

Suppose that $h$ satisfies \eqref{eq:scale-equivariance-22} for the
sample $S$. For the transformed sample, evaluate the right-hand side at
$ch$. From
\eqref{eq:scale-equivariance-10} and \eqref{eq:scale-equivariance-21},
\begin{equation}
\begin{aligned}
    &
    \left[
        \frac{D_n}{
            \widehat{\psi}_4
            \bigl(
                p_{S_{\delta,c}}(ch);
                S_{\delta,c}
            \bigr)
        }
    \right]^{1/5}
    \\
    &\qquad=
    \left[
        \frac{D_n}{
            \widehat{\psi}_4
            \bigl(
                c\,p_S(h);
                S_{\delta,c}
            \bigr)
        }
    \right]^{1/5}
    \\
    &\qquad=
    \left[
        \frac{D_n}{
            c^{-5}
            \widehat{\psi}_4
            \bigl(
                p_S(h);S
            \bigr)
        }
    \right]^{1/5}
    \\
    &\qquad=
    c
    \left[
        \frac{D_n}{
            \widehat{\psi}_4
            \bigl(
                p_S(h);S
            \bigr)
        }
    \right]^{1/5}
    \\
    &\qquad=
    ch.
\end{aligned}
\end{equation}
Thus, whenever $h$ solves the Sheather--Jones equation for $S$, the
bandwidth $ch$ solves the corresponding equation for $S_{\delta,c}$.
Consequently, the sets of positive solutions satisfy
\begin{equation}
    \mathcal{H}_{\mathrm{SJ}}(S_{\delta,c})
    =
    c\,\mathcal{H}_{\mathrm{SJ}}(S).
    \label{eq:scale-equivariance-23}
\end{equation}
When the defining equation has a unique positive solution, this reduces
to
\begin{equation}
    \boxed{
        h_{\mathrm{SJ}}(\delta+cS)
        =
        c\,h_{\mathrm{SJ}}(S).
    }
    \label{eq:scale-equivariance-24}
\end{equation}

\subsection{Least-Squares Cross-Validation on a Relative Grid}

For a bandwidth $h$, the least-squares cross-validation criterion is
\begin{equation}
    \operatorname{CV}(h;S)
    =
    \int_{\mathbb{R}}
    q_h(x;S)^2\,dx
    -
    \frac{2}{n}
    \sum_{i=1}^n
    q_{h,-i}(x_i;S),
    \label{eq:scale-equivariance-25}
\end{equation}
where
\begin{equation}
    q_{h,-i}(x;S)
    =
    \frac{1}{(n-1)h}
    \sum_{j\ne i}
    \phi
    \left(
        \frac{x-x_j}{h}
    \right)
\end{equation}
is the leave-one-out KDE.

The KDE transformation in \eqref{eq:scale-equivariance-1} gives
\begin{equation}
    q_{ch}(\delta+cy;S_{\delta,c})
    =
    c^{-1}q_h(y;S).
    \label{eq:scale-equivariance-26}
\end{equation}
The leave-one-out estimator satisfies the analogous relation
\begin{equation}
    q_{ch,-i}(\delta+cx_i;S_{\delta,c})
    =
    c^{-1}q_{h,-i}(x_i;S).
    \label{eq:scale-equivariance-27}
\end{equation}

For the integrated-square term, the change of variables
$x=\delta+cy$ gives
\begin{equation}
\begin{aligned}
    \int_{\mathbb{R}}
    q_{ch}(x;S_{\delta,c})^2\,dx
    &=
    \int_{\mathbb{R}}
    q_{ch}(\delta+cy;S_{\delta,c})^2c\,dy
    \\
    &=
    \int_{\mathbb{R}}
    c^{-2}q_h(y;S)^2c\,dy
    \\
    &=
    c^{-1}
    \int_{\mathbb{R}}
    q_h(y;S)^2\,dy.
\end{aligned}
\label{eq:scale-equivariance-28}
\end{equation}
The leave-one-out term satisfies
\begin{equation}
\begin{aligned}
    \frac{2}{n}
    \sum_{i=1}^n
    q_{ch,-i}(\delta+cx_i;S_{\delta,c})
    &=
    c^{-1}
    \frac{2}{n}
    \sum_{i=1}^n
    q_{h,-i}(x_i;S).
\end{aligned}
\label{eq:scale-equivariance-29}
\end{equation}
Combining \eqref{eq:scale-equivariance-28} and \eqref{eq:scale-equivariance-29} yields
\begin{equation}
    \boxed{
        \operatorname{CV}
        \bigl(
            ch;S_{\delta,c}
        \bigr)
        =
        c^{-1}
        \operatorname{CV}(h;S).
    }
    \label{eq:scale-equivariance-30}
\end{equation}

The multiplicative factor $c^{-1}$ is positive and independent of the
bandwidth. It therefore preserves the ordering of candidate
bandwidths.

In the present implementation, LSCV is evaluated over the grid
\begin{equation}
    h_k(S)
    =
    0.2h_{\mathrm{S}}(S)
    \left(
        \frac{5}{0.2}
    \right)^{k/59},
    \qquad
    k=0,\ldots,59.
    \label{eq:scale-equivariance-31}
\end{equation}
Because Silverman's bandwidth satisfies \eqref{eq:scale-equivariance-8},
\begin{equation}
\begin{aligned}
    h_k(S_{\delta,c})
    &=
    0.2h_{\mathrm{S}}(S_{\delta,c})
    \left(
        \frac{5}{0.2}
    \right)^{k/59}
    \\
    &=
    c\,h_k(S).
\end{aligned}
\label{eq:scale-equivariance-32}
\end{equation}
Thus, the candidate sets satisfy
\begin{equation}
    \mathcal{G}(S_{\delta,c})
    =
    c\,\mathcal{G}(S),
    \label{eq:scale-equivariance-33}
\end{equation}
where
\begin{equation}
    \mathcal{G}(S)
    =
    \{
        h_0(S),\ldots,h_{59}(S)
    \}.
\end{equation}

For every corresponding grid point,
\begin{equation}
    \operatorname{CV}
    \bigl(
        h_k(S_{\delta,c});S_{\delta,c}
    \bigr)
    =
    c^{-1}
    \operatorname{CV}
    \bigl(
        h_k(S);S
    \bigr).
    \label{eq:scale-equivariance-34}
\end{equation}
All $60$ criterion values are therefore multiplied by the same
positive constant, so the minimizing grid index is unchanged. It
follows that
\begin{equation}
    \boxed{
        h_{\mathrm{LSCV}}(\delta+cS)
        =
        c\,h_{\mathrm{LSCV}}(S).
    }
    \label{eq:scale-equivariance-35}
\end{equation}
If several grid points attain the same minimum, the corresponding
minimizer sets transform in the same way. A deterministic tie-breaking
rule based on the grid index also preserves the result.

\subsection{Consequence for the Gaussian Comparison}

Let
\begin{equation}
    S_0\sim\mathcal{N}(0,1)^n
\end{equation}
and define
\begin{equation}
    S_{\delta,c}=\delta+cS_0
    \sim
    \mathcal{N}(\delta,c^2)^n.
\end{equation}
For every selector considered in the Gaussian comparison,
\begin{equation}
    h_b(S_{\delta,c})
    =
    c\,h_b(S_0),
    \qquad
    b\in
    \{
        \theta,
        \mathrm{S},
        \mathrm{SJ},
        \mathrm{LSCV}
    \}.
\end{equation}
Equation \eqref{eq:scale-equivariance-2} then gives
\begin{equation}
    \mathcal{L}
    \bigl(
        h_b(S_{\delta,c});
        S_{\delta,c},
        f_{\delta,c}
    \bigr)
    =
    \mathcal{L}
    \bigl(
        h_b(S_0);
        S_0,
        \phi
    \bigr)
    +
    \log_2 c.
    \label{eq:scale-equivariance-36}
\end{equation}

The absolute logarithmic scores for
$\mathcal{N}(\delta,c^2)$ therefore differ from those for
$\mathcal{N}(0,1)$ by the common additive term $\log_2 c$. The
ordering of the selectors and all pairwise score differences are
unchanged. Standard-normal sampling is consequently sufficient for the
comparison of these bandwidth selectors throughout the Gaussian
location--scale family.

\section{Neural Network Architecture and Training Procedure}
\label{app:network_training}

\subsection{Gaussian Benchmark}

For the Gaussian benchmark, the amortized selector uses the scale-equivariant
parameterization
\[
    h_\theta(S)
    =
    \operatorname{sd}(S)\,F_\theta(n),
    \qquad
    F_\theta(n)>0,
\]
where $F_\theta(n)$ is predicted from the sample size alone. The predictor is a
fully connected multilayer perceptron with layer widths
\[
    1 \longrightarrow 4 \longrightarrow 4 \longrightarrow 1,
\]
with ReLU activations in the hidden layers and a softplus output transformation.

Training tasks are generated online. At each optimization step,
$n\sim\operatorname{Unif}\{5,\ldots,256\}$ is drawn, and a batch of 256
independent samples from $\mathcal{N}(0,1)^n$ is generated. The expected
logarithmic score for each task is evaluated using 256-point Gauss--Hermite
quadrature.

The network is optimized using Adam with learning rate $10^{-3}$ and gradient
clipping at $1$, for at most 10,000 optimization steps. The training seed is fixed
at 1234. Model selection is based on a fixed validation set over
\[
    \{5,10,15,20,30,50,75,100,150,200,256\},
\]
with 512 independent tasks at each sample size. Validation is performed every
500 steps. Early stopping is applied after 3000 steps and is triggered after three
consecutive validation evaluations without an improvement of at least
$10^{-4}$ bits. The checkpoint used in the Gaussian comparison was retained at
optimization step 3500.

\subsection{Bounded Multi-Family Benchmark}

For the bounded multi-family benchmark, the amortized selector uses the
five-dimensional feature representation in equation~(4). The predictor is a fully connected
network with layer widths
\[
    5 \longrightarrow 128 \longrightarrow 128 \longrightarrow 1,
\]
with ReLU hidden activations and a softplus output transformation to ensure a
positive bandwidth.

As described in Section~3.2 and Table~1 of the main paper, 10,000 parameter settings are generated
in advance for each of the ten distribution families. For each training task, a family
is sampled with equal probability, a parameter setting is drawn from the corresponding
family-specific pool, and a fresh sample is generated from the resulting bounded
distribution. The sample size is drawn independently for each task from
$n\sim\operatorname{Unif}\{5,\ldots,256\}$. Each training batch contains 256
tasks, and the expected bounded logarithmic score for each task is approximated
using 1024 independent Monte Carlo draws from the same target distribution.

Training uses Adam with learning rate $5\times10^{-4}$, gradient clipping at $1$,
and an exponential moving average of the network parameters with decay $0.999$.
Training is allowed to continue for at most 40,000 optimization steps, with training
seed 2026.

Model selection uses a fixed validation set containing 20 tasks for every integer
$n\in\{5,\ldots,256\}$, giving 5040 tasks in total. The validation set is generated
once with seed 777 and held fixed throughout training. Validation is performed
every 1000 steps. Early stopping is not applied before 12,000 steps and is triggered
after five consecutive validation evaluations without an improvement of at least
$5\times10^{-5}$ bits. Training reached the maximum of 40,000 optimization
steps, at which the best checkpoint was retained. The corresponding
exponential-moving-average weights are used for subsequent evaluation.

\subsection{GMM Benchmarks}

The $K=8$, $K=16$, and $K=32$ GMM-trained selectors use the same feature
representation, network architecture, bounded logarithmic score objective, and
training procedure as the multi-family selector.

The main difference is the generation of training tasks. Rather than sampling from
a fixed parameter pool, a fresh GMM is independently generated for every training
task using the corresponding generator defined in Section~3.3 of the main paper.
The sample size is drawn independently for each task from
$n\sim\operatorname{Unif}\{5,\ldots,256\}$, and both the observed
sample and an independent Monte Carlo evaluation sample of size 1024 are
generated from the same bounded GMM.

Separate networks are trained for $K=8$, $K=16$, and $K=32$. The optimizer,
batch size, exponential moving average, validation schedule, early-stopping rule,
and random seeds are the same as in the multi-family experiment. For each value of
$K$, model selection uses a separately generated fixed validation set containing
20 tasks for every integer $n\in\{5,\ldots,256\}$, giving 5040 validation tasks
in total. All three training runs reached the maximum of 40,000 optimization
steps without triggering early stopping. The best checkpoints were retained at
steps 39,000, 39,000, and 40,000 for $K=8$, $K=16$, and $K=32$, respectively,
and their exponential-moving-average weights are used in the final comparisons.

\subsection{Evaluation and Bootstrap Procedure}

All bandwidth-comparison experiments are evaluated on the sample-size grid
\[
\begin{split}
    \{&
    5,6,7,8,9,10,11,12,13,14,15,16,17,18,19,20,\\
    &25,30,35,40,50,60,75,90,110,130,150,175,200,225,256
    \}.
\end{split}
\]

For the Gaussian benchmark, 30,000 independent test samples are generated
at each sample size. For the bounded multi-family benchmark, 3000
independent test tasks are generated from each of the ten families at
each sample size. For the bounded GMM benchmark, 30,000 independently
generated $K=32$ GMM tasks are used at each sample size. The
cross-distribution comparison on the multi-family benchmark uses the
same balanced design of 3000 tasks per family.

For the bounded bandwidth-comparison experiments, each task uses 1024
independent Monte Carlo evaluation draws. Within each experiment, all
bandwidth selectors are evaluated using the same task-specific samples and
evaluation draws.

Uncertainty intervals are constructed using 1000 paired nonparametric bootstrap
replicates. At each sample size, each replicate resamples 30,000 task indices
with replacement, using the same indices for all methods. For the multi-family
benchmarks, resampling is performed over the pooled collection of 30,000 tasks,
without stratification by family. Thus, the reported empirical means weight
the ten families equally, while the family proportions may vary across
bootstrap replicates. The 5th and 95th percentiles of the bootstrap means
give the reported 90\% intervals.

\bibliographystyle{imsart-nameyear}
\bibliography{reference}